\documentclass{article}

\usepackage[preprint]{neurips_2026}

\usepackage[utf8]{inputenc} 
\usepackage[T1]{fontenc}    
\usepackage{hyperref}       
\usepackage{url}            
\usepackage{booktabs}       
\usepackage{amsfonts}       
\usepackage{nicefrac}       
\usepackage{microtype}      
\usepackage{xcolor}         
\usepackage{amsmath}
\usepackage{amssymb}
\usepackage{mathtools}
\usepackage{amsthm}
\usepackage{comment}
\usepackage{enumitem}
\usepackage[capitalize,noabbrev]{cleveref}
\usepackage{thmtools,thm-restate}
\usepackage{algorithm}
\usepackage{algorithmic}
\usepackage{multirow}

\usepackage{pifont}
\newcommand{\cmark}{\ding{51}}
\newcommand{\xmark}{\ding{55}}

\theoremstyle{plain}
\newtheorem{theorem}{Theorem}
\newtheorem{proposition}[theorem]{Proposition}

\theoremstyle{definition}

\theoremstyle{remark}

\title{Rethinking Importance Sampling in LLM Policy Optimization: A Cumulative Token Perspective}

\author{%
  Yuheng Zhang \thanks{Email: \texttt{yuhengz2@illinois.edu}.} \\
  UIUC \\
  \And
  Chenlu Ye \\
  UIUC \\
  \And
  Shuowei Jin \\
  University of Michigan \\
  \And
  Changlong Yu \\
  Amazon \\
  \AND 
  Wei Xiong \\
  UIUC \\
  \And
  Saurabh Sahu \\
  Amazon \\
  \And
  Nan Jiang \\
  UIUC \\
}

\begin{document}

\maketitle

\begin{abstract}
Reinforcement learning, including reinforcement learning with verifiable rewards (RLVR), has emerged as a powerful approach for LLM post-training. Central to these approaches is the design of the importance sampling (IS) ratio used in off-policy policy-gradient estimation. Existing methods face a fundamental bias-variance dilemma: token-level IS ratios, as adopted by PPO~\citep{schulman2017proximal} and GRPO~\citep{shao2024deepseekmath}, introduce bias by ignoring prefix state distribution mismatch; full sequence ratios provide exact trajectory-level correction but suffer 
from high variance due to the multiplicative accumulation of per-token ratios, while GSPO~\citep{zheng2025group} improves numerical stability via length normalization at the cost of deviating from the exact full-sequence IS correction. In this work, we identify the cumulative token IS ratio, the product of per-token ratios up to position $t$, as a theoretically principled solution to this dilemma. We prove that, under the token-level policy-gradient formulation, this ratio provides an unbiased prefix correction for each token-level gradient term and has strictly lower variance than the full sequence ratio. Building on this insight, we propose CTPO (Cumulative Token Policy Optimization), which combines the cumulative token IS ratio with position-adaptive clipping that scales log-space clip bounds according to the natural $\sqrt{t}$ growth of the cumulative log-ratio. This yields more consistent regularization across token positions. We implement and evaluate CTPO in the tool-integrated reasoning setting on several challenging mathematical reasoning benchmarks, achieving the best average performance across both model scales compared with strong GRPO and GSPO baselines. Code will be available at \url{https://github.com/horizon-llm/CTPO}.
\end{abstract}

\section{Introduction}
Reinforcement learning has emerged as a powerful paradigm for LLM 
post-training, achieving remarkable success on tasks with verifiable 
rewards such as mathematical reasoning and code generation. Landmark 
systems such as OpenAI o1~\citep{jaech2024openai} and DeepSeek-R1~\citep{guo2025deepseek} have demonstrated that reinforcement learning can dramatically enhance the reasoning capabilities of large language models. A standard algorithmic framework underlying these advances is policy gradient optimization, with Group Relative Policy Optimization 
(GRPO)~\citep{shao2024deepseekmath} serving as a representative 
example. Subsequent works have proposed numerous variants of GRPO, 
including DAPO~\citep{yu2025dapo} and Dr.~GRPO~\citep{liu2025understanding}, addressing issues such as sampling efficiency and reward normalization bias.

A central challenge in these policy gradient methods is handling the 
off-policy nature of training, where trajectories are collected from 
a behavior policy $\pi_b$ but used to update a different target policy 
$\pi_\theta$. Importance sampling (IS) is the standard tool for 
correcting this distribution mismatch. However, existing methods 
adopt IS ratio formulations with fundamental bias-variance limitations. 
GRPO uses token-level IS ratios 
$r_t = \pi_\theta(a_t \mid s_t) / \pi_b(a_t \mid s_t)$, 
which account only for the current-token action probability while 
ignoring the preceding state distribution mismatch, introducing a 
systematic bias into the gradient estimate. A full sequence ratio would 
correct the trajectory-level mismatch, but it suffers from high variance 
due to the multiplicative accumulation of per-token ratios across the 
entire response. GSPO~\citep{zheng2025group} improves numerical stability 
by using a length-normalized sequence-level ratio, equivalently the 
geometric mean of per-token ratios, but this normalized ratio no longer 
corresponds to the exact full-sequence IS correction and is therefore 
biased as an IS correction.

In this work, we identify the cumulative token IS ratio as a theoretically principled solution to this bias-variance dilemma. The cumulative token IS ratio at 
position $t$ is defined as the product of per-token ratios up to position 
$t$, i.e., $\rho_t^{\mathrm{cum}} = \prod_{t'=1}^{t} r_{t'}$. We prove that, 
under the token-level policy-gradient formulation, this ratio provides 
an unbiased prefix correction for each token-level gradient term, since 
it correctly accounts for the prefix trajectory likelihood under both 
policies. At the same time, it has strictly lower variance than the full 
sequence ratio, as suffix tokens beyond position $t$ have unit expectation 
under the behavior policy and contribute only unnecessary variance without 
reducing bias. This makes the cumulative token IS ratio a principled middle ground: it preserves the exact prefix correction needed at each token 
position while avoiding the unnecessary suffix variance of the full 
sequence ratio.

Building on this insight, we further identify a practical challenge 
that arises when applying clipping to the cumulative token IS ratio. 
Since $\log \rho_t^{\mathrm{cum}} = \sum_{t'=1}^{t} \log r_{t'}$ 
accumulates per-token log-ratios along the sequence, its variance 
grows linearly with position: $\mathrm{Var}(\log \rho_t^{\mathrm{cum}}) 
= t\sigma^2$. As a result, a fixed clipping range leads to 
inconsistent regularization across token positions — early tokens 
are rarely clipped while late tokens are clipped far more frequently. 
We empirically observe this log-space variance growth during training and 
propose a position-adaptive clipping strategy that scales the clipping 
thresholds proportionally to $\sqrt{t}$, matching the natural standard 
deviation growth of the cumulative log-ratio and providing more consistent 
regularization across token positions.

We instantiate these ideas in \textsc{CTPO} (Cumulative Token Policy 
Optimization) and evaluate it in the tool-integrated reasoning (TIR) 
setting~\citep{xue2025simpletir}, where the model iteratively generates 
and executes Python code in a sandboxed environment to solve challenging 
mathematical problems. Our contributions are summarized as follows:

\begin{itemize}
    \item \textbf{Cumulative Token IS Ratio.} We identify the cumulative 
    token IS ratio as a theoretically principled solution to the bias-variance 
    dilemma in off-policy LLM post-training. Under the token-level 
    policy-gradient formulation, we prove that it provides an unbiased prefix 
    correction for each gradient term and has strictly lower 
    variance than the full sequence ratio.

    \item \textbf{Position-Adaptive Clipping.} We propose a 
    position-adaptive clipping strategy based on the observation that 
    $\mathrm{Var}(\log \rho_t^{\mathrm{cum}})$ grows linearly with 
    token position. By scaling the clipping thresholds proportionally 
    to $\sqrt{t}$, our strategy provides more uniform regularization 
    across token positions compared to fixed clipping.

    \item \textbf{Empirical Validation.} We evaluate \textsc{CTPO} in 
    the tool-integrated reasoning setting on challenging mathematical 
    benchmarks, where it achieves the best average performance across 
    both model scales compared with strong GRPO and GSPO baselines. 
    Ablation studies further confirm the contribution of position-adaptive 
    clipping to overall performance.
\end{itemize}
\section{Preliminaries}

We first review the token-level MDP formulation underlying LLM generation, 
then present GRPO and GSPO as two representative baselines that adopt 
different importance sampling strategies.

\paragraph{Token-level MDP.}
LLM text generation can be formulated as a finite-horizon Markov Decision 
Process (MDP) $(\mathcal{S}, \mathcal{V}, P, R, H)$, where the state $s_t$ 
at step $t$ is the concatenation of the prompt $x$ and all previously 
generated tokens $(a_1, \ldots, a_{t-1})$, the action $a_t \in \mathcal{V}$ 
is the next token drawn from the policy $\pi_\theta(\cdot \mid s_t)$, and the 
transition is deterministic: $s_{t+1} = s_t \circ a_t$. A scalar reward 
$R(\tau)$ is assigned to the complete trajectory 
$\tau = (s_1, a_1, \ldots, s_H, a_H)$ upon termination.

\paragraph{Policy Gradient and Importance Sampling.}
The learning objective is to maximize the expected reward 
$J(\theta) = \mathbb{E}_{\tau \sim \pi_\theta}[R(\tau)]$. 
By the policy gradient theorem, the gradient is:
\begin{equation*}
    \nabla_\theta J(\theta) = \mathbb{E}_{\tau \sim \pi_\theta}
    \left[\sum_{t=1}^{H} A_t(s_t, a_t)\, \nabla_\theta \log \pi_\theta(a_t \mid s_t)\right],
\end{equation*}
where $A_t(s_t,a_t)$ is the advantage function at position $t$. 
In practice, trajectories are collected from a fixed behavior policy 
$\pi_b = \pi_{\theta_{\text{old}}}$ and reused across multiple gradient steps. 
To correct for this distribution mismatch, importance sampling (IS) yields:
\begin{equation*}
    \nabla_\theta J(\theta) = \mathbb{E}_{\tau \sim \pi_b}
    \left[\sum_{t=1}^{H} \rho_t\, A_t(s_t, a_t)\, \nabla_\theta \log \pi_\theta(a_t \mid s_t)\right],
\end{equation*}
where $\rho_t$ is the IS ratio at position $t$. 
The choice of $\rho_t$ is central to the bias-variance tradeoff 
and is the main distinction between existing algorithms.

\subsection{Group Relative Policy Optimization (GRPO)}

GRPO~\citep{shao2024deepseekmath} eliminates the value network of PPO by 
estimating advantages through group-relative reward normalization. 
For each prompt $x$, the behavior policy $\pi_{\theta_{\text{old}}}$ generates 
$G$ responses $\{o_i\}_{i=1}^{G}$. The advantage of response $o_i$ is computed as:
\begin{equation*}
    A_i = \frac{R_i - \mathrm{mean}(\{R_j\}_{j=1}^{G})}{\mathrm{std}(\{R_j\}_{j=1}^{G})},
\end{equation*}
where $R_i = \mathcal{R}(x, o_i)$ is the scalar reward assigned to 
the entire response $o_i$. Since the reward is only available at the sequence level, GRPO assigns the same advantage to every token in $o_i$, i.e., $A_t(s_t, a_t) = A_i$ for all $t \in \{1, \ldots, |o_i|\}$. GRPO then maximizes the following clipped surrogate objective:
\begin{equation*}
    \mathcal{J}_{\mathrm{GRPO}}(\theta) = \mathbb{E}
    \left[\frac{1}{G}\sum_{i=1}^{G} \frac{1}{|o_i|} \sum_{t=1}^{|o_i|}
    \min\!\left(r_{i,t}\, A_i,\ \mathrm{clip}(r_{i,t}, 1-\varepsilon, 1+\varepsilon)\, A_i\right)\right],
\end{equation*}
where the token-level IS ratio is:
\begin{equation*}
    r_{i,t} = \frac{\pi_\theta(o_{i,t} \mid x, o_{i,<t})}{\pi_{\theta_{\text{old}}}(o_{i,t} \mid x, o_{i,<t})}.
\end{equation*}

\paragraph{Bias of the token-level IS ratio.}
While $r_{i,t}$ is straightforward to compute, it constitutes a 
\emph{biased} estimator of the true IS weight at position $t$. 
The theoretically correct IS ratio for correcting the distribution 
mismatch of the state-action pair $(s_t,a_t)$ must account for both 
the probability of reaching $s_t$ and sampling $a_t$ under 
$\pi_\theta$ versus $\pi_b$, which involves the prefix trajectory up 
to position $t$:
\begin{equation*}
    \rho_t^* = \frac{\pi_\theta(o_{i,1:t} \mid x)}{\pi_b(o_{i,1:t} \mid x)} 
    = \prod_{t'=1}^{t} r_{i,t'}.
\end{equation*}
GRPO uses only the current-token ratio $r_{i,t}$, 
discarding the prefix $\prod_{t'=1}^{t-1} r_{i,t'}$. 
This simplification introduces a systematic bias in the gradient estimate, 
though it yields lower variance.

\subsection{Group Sequence Policy Optimization (GSPO)}
A natural way to address the token-level bias is to define the IS ratio 
at the sequence level. The full sequence ratio for response $o_i$ is:
\begin{equation*}
    \rho_i^{\mathrm{seq}} = 
    \frac{\pi_\theta(o_i \mid x)}{\pi_b(o_i \mid x)} 
    = \prod_{t=1}^{|o_i|} r_{i,t}.
\end{equation*}
This full sequence ratio corrects the trajectory-level distribution 
mismatch, but it can have high variance due to the multiplicative 
accumulation of per-token ratios across the entire response. To mitigate the high variance, GSPO~\citep{zheng2025group} uses a length-normalized sequence-level ratio:
\begin{equation*}
    \rho_i^{\mathrm{GSPO}} = 
    \left(
    \frac{\pi_\theta(o_i \mid x)}{\pi_b(o_i \mid x)}
    \right)^{1/|o_i|}
    =
    \left(\prod_{t=1}^{|o_i|} r_{i,t}\right)^{1/|o_i|}.
\end{equation*}
This ratio is equivalently the geometric mean of the per-token ratios within the sequence. It is applied uniformly to the sequence, and GSPO clips at the sequence 
level rather than the token level:
\begin{equation*}
    \mathcal{J}_{\mathrm{GSPO}}(\theta) = \mathbb{E}
    \left[\frac{1}{G}\sum_{i=1}^{G} 
    \min\!\left(\rho_i^{\mathrm{GSPO}}\, A_i,\ 
    \mathrm{clip}(\rho_i^{\mathrm{GSPO}}, 1-\varepsilon, 1+\varepsilon)\, A_i\right)\right].
\end{equation*}
Length normalization reduces the scale of the IS ratio and improves 
numerical stability, but the resulting ratio no longer corresponds to 
the exact full-sequence IS correction.

In summary, the token-level ratio and the full sequence ratio represent 
two ends of the bias-variance spectrum: the token-level ratio is biased 
but low-variance, while the full sequence ratio is unbiased but 
high-variance. GSPO reduces variance through length normalization, but 
its normalized sequence-level ratio is still biased as an IS correction.
In the next section, we show that the cumulative token IS ratio 
resolves this dilemma by achieving both unbiasedness and lower variance 
than the full sequence formulation.

\section{Cumulative Token Policy Optimization}\label{sec:ctpo}

\subsection{Cumulative Token Importance Ratio}

Recall that the policy gradient requires an expectation under the 
current policy $\pi_\theta$, but in practice trajectories are sampled 
from the behavior policy $\pi_b$. Throughout this section, we analyze 
importance-ratio correction for a generic token-level advantage function 
$A_t(s_t,a_t)$. To correct the distribution mismatch for the state-action 
pair $(s_t,a_t)$ at position $t$, the IS weight must account for the 
probability of reaching $s_t$ and sampling $a_t$ under $\pi_\theta$ 
relative to $\pi_b$. Since the transition dynamics are deterministic, 
this reduces to the ratio of the prefix trajectory likelihoods:
\begin{equation*}
    \rho_t^{\mathrm{cum}} = \frac{\pi_\theta(a_{1:t} \mid x)}{\pi_b(a_{1:t} \mid x)} 
    = \prod_{t'=1}^{t} \frac{\pi_\theta(a_{t'} \mid s_{t'})}{\pi_b(a_{t'} \mid s_{t'})} 
    = \prod_{t'=1}^{t} r_{t'}.
\end{equation*}
We term this the \emph{cumulative token importance ratio}: unlike GRPO, 
which uses only the current-token ratio $r_t$ and ignores the prefix 
trajectory, and unlike the full sequence ratio 
$\prod_{t'=1}^{H} r_{t'}$, which includes the entire response, 
$\rho_t^{\mathrm{cum}}$ accumulates per-token ratios only up to position $t$, 
naturally matching the temporal structure of the MDP. 
Intuitively, $\rho_t^{\mathrm{cum}}$ asks: how much more or less likely 
is it that $\pi_\theta$ would have produced the prefix $(a_1, \ldots, a_t)$ 
compared to $\pi_b$? This is precisely the prefix correction factor needed 
for the token-level policy-gradient term at position $t$, as we formalize 
in the following proposition.

\begin{proposition}[Unbiasedness of Cumulative Token IS Ratio]
\label{prop:unbiased}
Under the token-level MDP formulation, let $A_t(s_t,a_t)$ denote the 
token-level advantage function, the gradient estimator using the cumulative 
token IS ratio is unbiased:
\begin{equation*}
    \mathbb{E}_{\tau \sim \pi_b}\left[\sum_{t=1}^{H} \rho_t^{\mathrm{cum}}\, 
    A_t(s_t, a_t)\, \nabla_\theta \log \pi_\theta(a_t \mid s_t)\right]
    = \nabla_\theta J(\theta).
\end{equation*}
\end{proposition}
The proof proceeds by applying a change of measure to the prefix 
trajectory $a_{1:t}$, after which marginalizing over the suffix 
$a_{t+1:H}$ introduces no additional correction factor since both 
$A_t$ and $\nabla_\theta \log \pi_\theta(a_t \mid s_t)$ depend only 
on $(x, a_{1:t})$. The full proof is provided in Appendix~\ref{app:proof_1}.


\subsection{Bias-Variance Analysis}
We now analyze the bias and variance properties of the considered IS ratio 
formulations. As we will show, the token-level ratio $r_t$ sacrifices 
prefix correction for low variance, the full sequence ratio 
$\rho^{\mathrm{seq}}$ is unbiased but suffers from high variance, and 
the cumulative token ratio $\rho_t^{\mathrm{cum}}$ preserves the exact 
prefix correction while avoiding unnecessary suffix variance.
\paragraph{Bias of the token-level ratio.}
The token-level ratio $r_t$ used by GRPO is a biased estimator 
of the true IS weight. 
The correct IS weight for position $t$ must account for the 
probability of reaching state $s_t$ under $\pi_\theta$ versus $\pi_b$, 
which requires the prefix ratio $\rho_t^{\mathrm{cum}}$. 
Formally, for $t>1$,
\begin{align*}
    \mathbb{E}_{\tau \sim \pi_b}\left[r_t\, A_t(s_t,a_t)\, 
    \nabla_\theta \log \pi_\theta(a_t \mid s_t)\right]
    &\neq 
    \mathbb{E}_{\tau \sim \pi_\theta}\left[A_t(s_t,a_t)\, 
    \nabla_\theta \log \pi_\theta(a_t \mid s_t)\right]
\end{align*}
in general, since $r_t$ only corrects the current-token action probability 
and omits the prefix correction $\prod_{t'=1}^{t-1} r_{t'}$ for the state 
distribution mismatch.

\paragraph{Variance of the sequence-level ratio.}
The sequence-level ratio $\rho^{\mathrm{seq}} = \prod_{t'=1}^{H} r_{t'}$ 
is unbiased but incurs unnecessarily high variance. 
To see this, observe that for any position $t$:
\begin{equation*}
    \rho^{\mathrm{seq}} = \rho_t^{\mathrm{cum}} \cdot \underbrace{\prod_{t'=t+1}^{H} r_{t'}}_{\epsilon_t},
\end{equation*}
where $\epsilon_t = \prod_{t'=t+1}^{H} r_{t'}$ is the suffix ratio 
from position $t+1$ to $H$. By the likelihood ratio identity:
\begin{equation*}
    \mathbb{E}_{\pi_b}[\epsilon_t \mid s_t, a_t] = 1,
\end{equation*}
so $\epsilon_t$ does not contribute to correcting the distribution 
mismatch at position $t$, yet it inflates the variance of the estimator. 
This motivates replacing $\rho^{\mathrm{seq}}$ with $\rho_t^{\mathrm{cum}}$, 
which retains all information necessary for unbiasedness while 
discarding the variance-inflating suffix.

We formalize this variance reduction in the following proposition.

\begin{proposition}[Variance Reduction of Cumulative Token IS Ratio]
\label{prop:variance}
The following statements hold regarding the variance of the IS ratios.
\textbf{(i)} For any $t < H$, if $\pi_\theta$ and $\pi_b$ differ on a $\pi_b$ reachable future state, then:
\begin{equation*}
    \mathrm{Var}(\rho^{\mathrm{seq}}) > \mathrm{Var}(\rho_t^{\mathrm{cum}}),
\end{equation*}
\textbf{(ii) (Independence assumption)} Further assume the per-token 
ratios $r_{t'}$ are independent across positions, 
and let 
$\chi^2_{t'} =
\mathbb{E}_{\pi_b}
\left[
\chi^2(\pi_\theta(\cdot \mid s_{t'}) \| \pi_b(\cdot \mid s_{t'}))
\right]$
denote the averaged local $\chi^2$ divergence at position $t'$. Then:
\begin{equation*}
    \mathrm{Var}(\rho_t^{\mathrm{cum}}) = \prod_{t'=1}^{t}(1 + \chi^2_{t'}) - 1,
    \qquad
    \mathrm{Var}(\rho^{\mathrm{seq}}) = \prod_{t'=1}^{H}(1 + \chi^2_{t'}) - 1,
\end{equation*}
and consequently:
\begin{equation*}
    \frac{\mathrm{Var}(\rho^{\mathrm{seq}})}{\mathrm{Var}(\rho_t^{\mathrm{cum}})} 
    = \frac{\prod_{t'=1}^{H}(1 + \chi^2_{t'}) - 1}{\prod_{t'=1}^{t}(1 + \chi^2_{t'}) - 1}.
\end{equation*}
\end{proposition}
The proof is deferred to Appendix~\ref{app:proof_2}. Part~(i) holds without any distributional assumptions, showing that the 
cumulative token IS ratio has strictly lower variance than the full 
sequence ratio whenever the future token ratios after position $t$ vary 
under the behavior policy. Part~(ii) 
further quantifies this reduction. As a concrete example, if all positions 
share the same averaged local divergence $\chi^2_{t'} = \delta$, the variance ratio 
equals $\frac{(1+\delta)^H - 1}{(1+\delta)^t - 1}$. In the near on-policy 
regime where $\pi_\theta$ and $\pi_b$ are close, $\delta \to 0$ and 
the ratio converges to $H/t$, meaning tokens at early positions benefit 
the most from using $\rho_t^{\mathrm{cum}}$ over $\rho^{\mathrm{seq}}$. 
As the degree of off-policy increases, $\delta$ grows and the ratio 
scales as $(1+\delta)^{H-t}$, increasing exponentially with the remaining 
sequence length $H - t$. This advantage grows with sequence length $H$, 
making the cumulative token IS ratio particularly well-suited for tasks 
that require extended generation.

We summarize the bias-variance properties of the considered IS ratio 
formulations in Table~\ref{tab:comparison}. GRPO uses a token-level ratio, 
which has low variance but omits prefix correction. The full sequence ratio 
is unbiased but suffers from high variance. GSPO reduces the scale of the 
importance weight through length normalization, but the normalized ratio is 
biased as an IS correction. CTPO achieves the desired combination among 
these formulations: it preserves the exact prefix correction while having 
strictly lower variance than the full sequence ratio.

\begin{table}[t]
\centering
\caption{Comparison of importance sampling ratio formulations in terms 
of bias and variance.}
\label{tab:comparison}
\begin{tabular}{lccc}
\toprule
\textbf{Method} & \textbf{IS Ratio} & \textbf{Unbiased} & \textbf{Variance} \\
\midrule
GRPO  & $r_t$ & \xmark & Low \\
Full sequence ratio & $\prod_{t'=1}^{H} r_{t'}$ & \cmark & High \\
GSPO (length-normalized) & $\left(\prod_{t'=1}^{H} r_{t'}\right)^{1/H}$ & \xmark & Reduced scale \\
CTPO (ours) & $\prod_{t'=1}^{t} r_{t'}$ & \cmark & \textbf{Lower than full sequence} \\
\bottomrule
\end{tabular}
\end{table}

\subsection{Position-Adaptive Clipping}\label{sec:position-adaptive}
Clipping is a standard technique in PPO and GRPO to stabilize training 
by constraining the IS ratio within a fixed range $[1-\varepsilon, 1+\varepsilon]$, 
preventing excessively large policy updates. However, directly applying 
this fixed clipping range to $\rho_t^{\mathrm{cum}}$ is inadequate. Recall that $\log \rho_t^{\mathrm{cum}} = \sum_{t'=1}^{t} \log r_{t'}$ 
accumulates the per-token log-ratios along the generated sequence, 
and as these shifts accumulate, its magnitude tends to grow with 
position $t$. To see this more precisely, assuming the per-token 
log-ratios are independent with variance $\sigma^2$, 
the variance of the cumulative sum grows linearly with position:
\begin{equation*}
    \mathrm{Var}(\log \rho_t^{\mathrm{cum}}) = t\sigma^2,
\end{equation*}
causing early tokens to be rarely clipped while later tokens with 
larger cumulative deviations are clipped far more frequently, 
leading to inconsistent regularization across positions.

To address this, we propose scaling the clipping thresholds in log-space 
proportionally to the standard deviation of $\log \rho_t^{\mathrm{cum}}$. 
Specifically, we define position-adaptive clipping thresholds as:
\begin{equation*}
    \varepsilon_{\mathrm{high}}(t) = \varepsilon_{\mathrm{high}} \cdot t^p, 
    \qquad 
    \varepsilon_{\mathrm{low}}(t) = \varepsilon_{\mathrm{low}} \cdot t^p,
\end{equation*}
where $\varepsilon_{\mathrm{high}}, \varepsilon_{\mathrm{low}} > 0$ are 
base thresholds and $p > 0$ is the scaling exponent, yielding the 
position-dependent trust region:
\begin{equation*}
    \rho_t^{\mathrm{cum}} \in 
\left[e^{-\varepsilon_{\mathrm{low}}(t)},\ 
e^{\varepsilon_{\mathrm{high}}(t)}\right].
\end{equation*}
With $p = 0.5$, the clipping thresholds grow as $\sqrt{t}$, matching the standard deviation growth of $\log \rho_t^{\mathrm{cum}}$ 
under the independence assumption, providing more consistent 
regularization across positions. 

In practice, following GRPO, we avoid training a separate critic network 
and use an outcome-level group-relative reward as a surrogate advantage 
estimate. Specifically, for each sampled response $o_i$, we compute a 
scalar advantage $A_i$ from the normalized outcome reward and apply it 
uniformly across all token positions. Combining this practical advantage 
estimation strategy with the cumulative token IS ratio and 
position-adaptive clipping, the final CTPO objective is:
\begin{equation*}
    \mathcal{J}_{\mathrm{CTPO}}(\theta) = \mathbb{E}
    \left[\frac{1}{G}\sum_{i=1}^{G} \frac{1}{|o_i|} \sum_{t=1}^{|o_i|}
    \min\!\left(\rho_{i,t}^{\mathrm{cum}}\, A_i,\ 
    \mathrm{clip}\!\left(\rho_{i,t}^{\mathrm{cum}}, 
    e^{-\varepsilon_{\mathrm{low}}(t)}, 
    e^{\varepsilon_{\mathrm{high}}(t)}\right) A_i\right)\right].
\end{equation*}

\section{Experiments}\label{sec:exp}

\subsection{Experimental Setup}

\paragraph{Task and Models.}
We focus on the tool-integrated reasoning (TIR) setting~\citep{xue2025simpletir}, 
where the model iteratively generates Python code, executes it in a sandboxed 
Python interpreter, and uses the returned output to inform subsequent reasoning 
steps. Each trajectory consists of up to 5 turns of model-environment interaction, 
with a maximum response length of 8{,}000 tokens per trajectory. We conduct experiments 
on two model scales: Qwen3-4B and Qwen3-14B~\citep{yang2025qwen3}.

\paragraph{Datasets and Benchmarks.}
We train on DeepScaleR~\citep{deepscaler2025}, a curated dataset of approximately 40K mathematics problems. We evaluate on four challenging competition-level mathematical reasoning 
benchmarks: AIME 2025\footnote{\url{https://huggingface.co/datasets/math-ai/aime25}}, 
AIME 2026\footnote{\url{https://huggingface.co/datasets/math-ai/aime26}}, 
HMMT 2025\footnote{\url{https://huggingface.co/datasets/MathArena/hmmt_nov_2025}}, 
and BRUMO 2025\footnote{\url{https://huggingface.co/datasets/MathArena/brumo_2025}}, spanning a range of olympiad-level competition problems. For each benchmark we report avg@32, where we sample 32 responses 
per problem and average the pass rate.

\paragraph{Baselines and Implementation.}
We compare CTPO against two baselines representing different IS ratio designs: 
GRPO~\citep{shao2024deepseekmath}, which adopts token-level IS ratios, 
and GSPO~\citep{zheng2025group}, which adopts a length-normalized sequence-level IS ratio. 
To isolate the effect of IS ratio design on downstream performance, 
we keep all other training configurations identical across methods: 
a prompt batch size of 512, 8 responses per prompt, and a learning rate 
of $1\times10^{-6}$, implemented using the VERL 
framework~\citep{sheng2025hybridflow}. For CTPO, we set the position-adaptive clipping hyperparameters to $\varepsilon_{\mathrm{low}}=0.025$ and 
$\varepsilon_{\mathrm{high}}=0.05$. All experiments are conducted on a single node of 8 NVIDIA H100 or H200 GPUs.

\subsection{Main Results}
\begin{table}[t]
\centering
\caption{Main results on competition-level mathematical reasoning benchmarks
in the tool-integrated reasoning setting. All models are evaluated with avg@32. \textsc{CTPO} achieves the best average performance across both model scales.}
\label{tab:main_results}
\renewcommand{\arraystretch}{1.2}
\begin{tabular}{llccccc}
\toprule
\textbf{Base Model} & \textbf{Method} & \textbf{AIME 25} & \textbf{AIME 26} 
& \textbf{BRUMO 25} & \textbf{HMMT 25} & \textbf{Avg} \\
\midrule
\multirow{4}{*}{Qwen3-4B}
    & Base                 &  1.9 &  3.4 &  6.3 &  4.1 &   3.9 \\
    & GRPO                 &   43.0 &   44.1 &   53.1 &   32.6 &    43.2 \\
    & GSPO                 & 49.2 & 47.6 & 56.8 & 37.0 &  47.7 \\
    & \textsc{CTPO} (Ours) & 53.5 & 52.1 & 59.7 & 40.4 & \textbf{51.4} \\
\midrule
\multirow{4}{*}{Qwen3-14B}
    & Base                 &  4.3 &  6.9 &  7.0 &  6.0 &   6.1 \\
    & GRPO                 & 55.3 & 49.8 & 61.9 & 51.5 &  54.6 \\
    & GSPO                 & 56.4 & 56.1 & 63.3 & 46.0 &  55.5 \\
    & \textsc{CTPO} (Ours) & 65.0 & 59.1 & 63.0 & 48.0 & \textbf{58.8} \\
\bottomrule
\end{tabular}
\end{table}

Table~\ref{tab:main_results} presents the main results of CTPO against 
GRPO and GSPO on four competition-level mathematical reasoning benchmarks 
in the tool-integrated reasoning setting.

CTPO consistently achieves the best average performance across both model 
scales. On Qwen3-4B, CTPO outperforms the strongest baseline GSPO by 3.7 
points on average (51.4 vs. 47.7), a relative improvement of 7.8\%. 
On Qwen3-14B, CTPO outperforms GSPO by 3.3 points (58.8 vs. 55.5), 
a relative improvement of 5.9\%. Gains are observed consistently across 
both model scales, demonstrating that the cumulative token IS ratio provides 
a reliable improvement over strong token-level and length-normalized sequence-level baselines.

These results support our central claim: by preserving the exact prefix 
correction needed at each token position while avoiding unnecessary suffix 
variance, the cumulative token IS ratio translates its theoretical advantage 
into consistent average improvements across challenging competition-level 
benchmarks.

\subsection{Cumulative IS Ratio Variance and Clipping Behavior}

In this subsection, we analyze the dynamics of the cumulative token IS ratio $\rho_t^{\mathrm{cum}}$ 
during training to empirically examine its  position-dependent variance growth and motivate the position-adaptive clipping design. All analyses in this section are based on Qwen3-4B.

\paragraph{Variance growth of $\log \rho_t^{\mathrm{cum}}$.}
The top row of Figure~\ref{fig:is_ratio_analysis} shows the empirical standard 
deviation of $\log \rho_t^{\mathrm{cum}}$ as a function of token position $t$ 
across training steps 50, 100, and 150. In all cases, the empirical std closely 
follows the fitted curve $\hat{\sigma}\sqrt{t}$, consistent with the 
theoretical calculation $\mathrm{Var}(\log \rho_t^{\mathrm{cum}}) = t\sigma^2$ 
from Section~\ref{sec:position-adaptive}. This variance growth with position confirms 
that a fixed clipping range becomes increasingly misaligned 
as token position grows.

\paragraph{Clip rate under fixed vs. adaptive clipping.}
The bottom row of Figure~\ref{fig:is_ratio_analysis} compares the per-position 
clip rate under fixed clipping (ratio $\in [0.5, 5]$) and adaptive clipping 
across the same training steps. The fixed clip rate increases monotonically 
with position $t$, reaching up to 20\% for late tokens while remaining near 
0\% for early tokens. This positional imbalance disproportionately 
suppresses gradient updates from late tokens, an issue that is further 
amplified in extended generation tasks. In contrast, adaptive clipping 
maintains a substantially more uniform clip rate across all positions, 
consistently around 5--10\% regardless of token position or training step. 
This demonstrates that position-adaptive clipping effectively compensates 
for the variance growth of $\log \rho_t^{\mathrm{cum}}$.

\begin{figure}[t]
    \centering
    \includegraphics[width=0.32\textwidth]{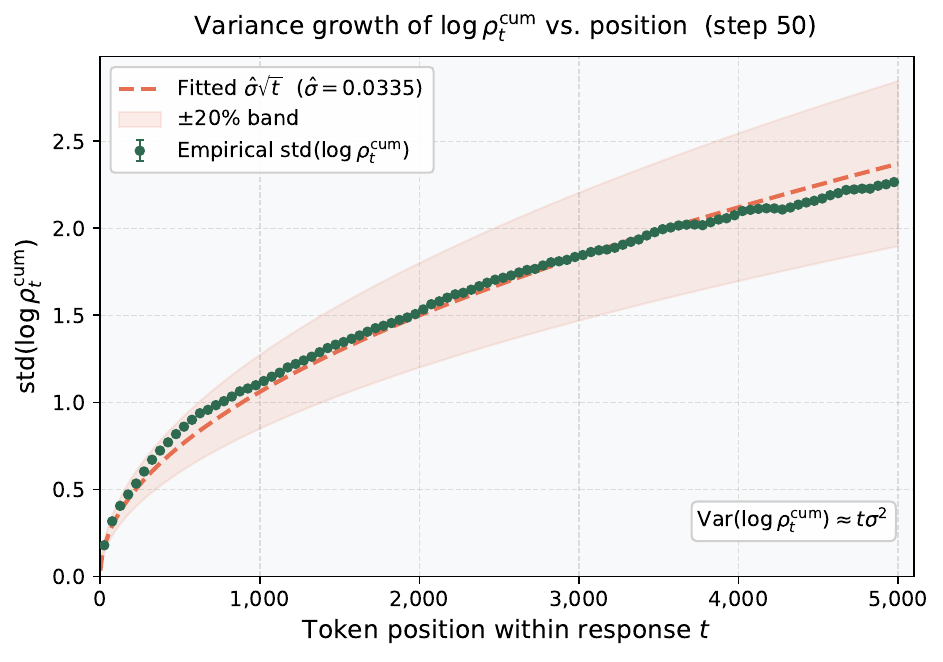}
    \hfill
    \includegraphics[width=0.32\textwidth]{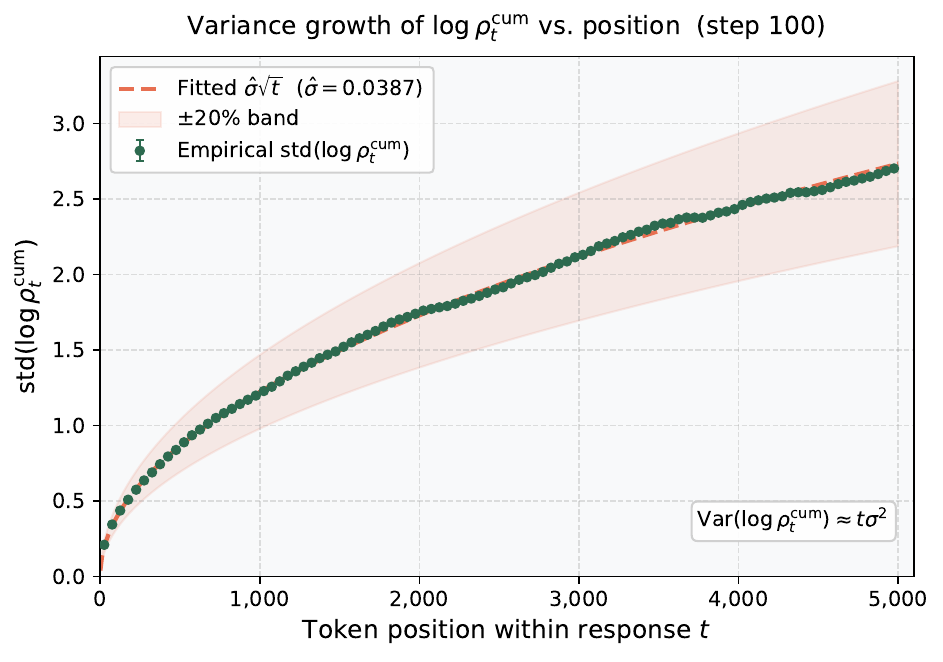}
    \hfill
    \includegraphics[width=0.32\textwidth]{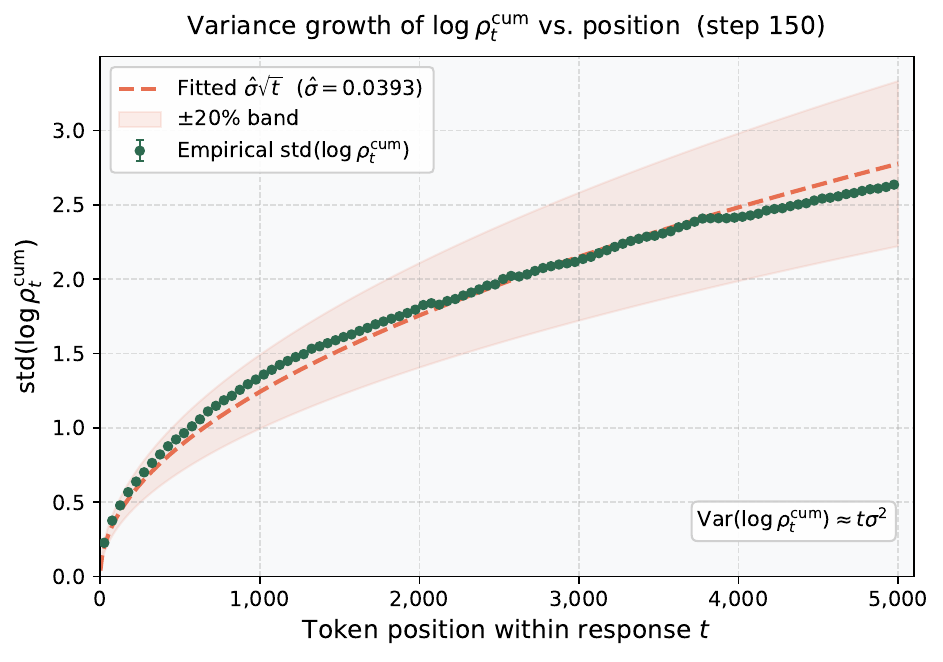}
    \includegraphics[width=0.32\textwidth]{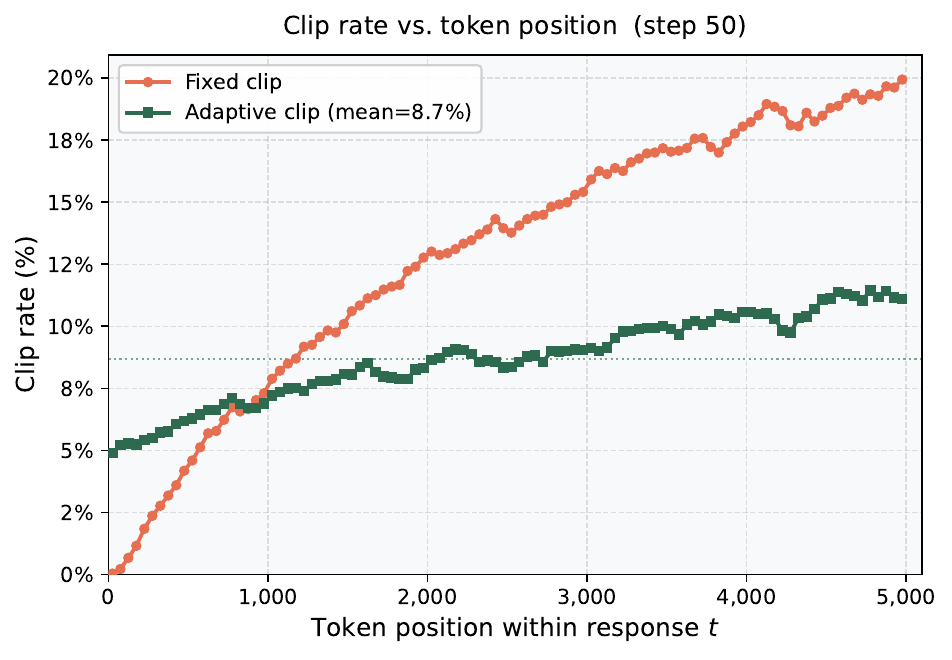}
    \hfill
    \includegraphics[width=0.32\textwidth]{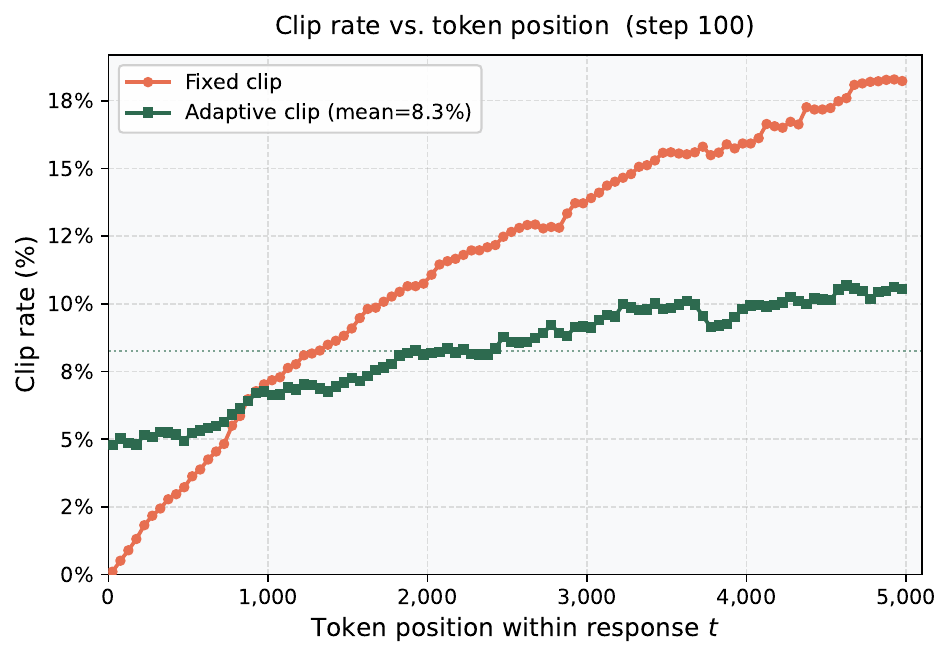}
    \hfill
    \includegraphics[width=0.32\textwidth]{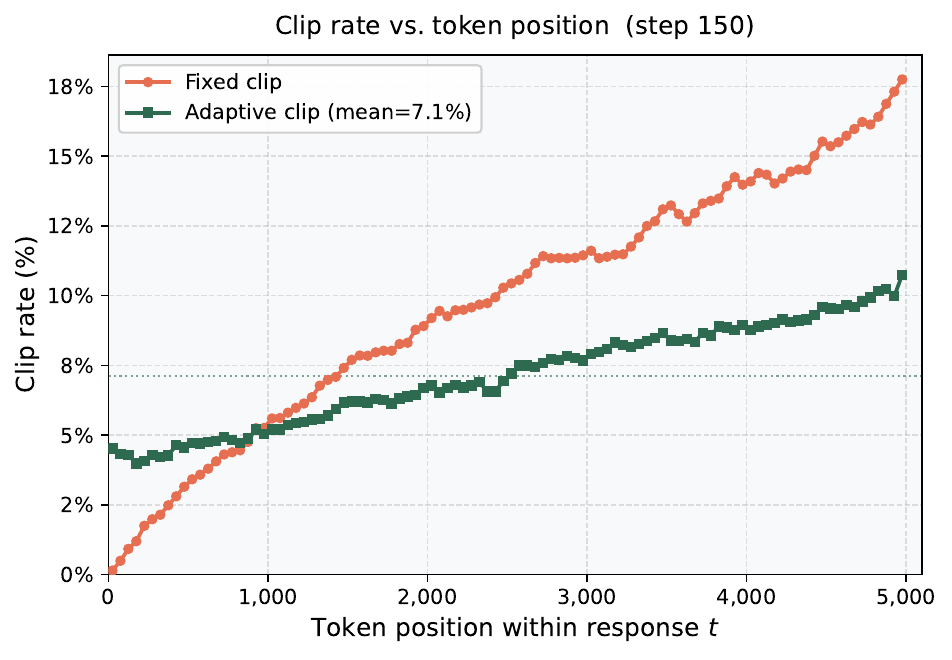}
    \caption{Analysis of $\log \rho_t^{\mathrm{cum}}$ across training steps 50, 100, and 150. \textbf{Top row:} Empirical standard deviation of $\log \rho_t^{\mathrm{cum}}$ vs. token position $t$, fitted with $\hat{\sigma}\sqrt{t}$, confirming the log-space variance growth discussed in 
    Section~\ref{sec:position-adaptive}. \textbf{Bottom row:} Clip rate vs. position under fixed clipping (ratio $\in [0.5, 5]$) and adaptive clipping. The fixed clip rate grows monotonically with $t$, while adaptive clipping maintains a substantially more uniform clip rate across positions.}
    \label{fig:is_ratio_analysis}
\end{figure}

\subsection{Ablation Study}
\begin{table}[t]
\centering
\caption{Ablation study on position-adaptive clipping vs. fixed clipping. 
Position-adaptive clipping consistently outperforms fixed clipping 
across all benchmarks.}
\label{tab:ablation}
\renewcommand{\arraystretch}{1.2}
\begin{tabular}{lccccc}
\toprule
\textbf{Method} & \textbf{AIME 25} & \textbf{AIME 26} 
& \textbf{BRUMO 25} & \textbf{HMMT 25} & \textbf{Avg} \\
\midrule
CTPO w/ fixed clip   & 55.5 & 57.5 & 62.6 & 47.1 & 55.7 \\
\textsc{CTPO} (Ours) & 65.0 & 59.1 & 63.0 & 48.0 & \textbf{58.8} \\
\bottomrule
\end{tabular}
\end{table}

Table~\ref{tab:ablation} ablates the contribution of position-adaptive 
clipping on Qwen3-14B by replacing it with a fixed clipping range while 
keeping all other components of CTPO unchanged. Position-adaptive clipping 
outperforms fixed clipping by 3.1 points on average (58.8 vs. 55.7), 
with consistent gains across all four benchmarks. These results indicate that the position-dependent variance growth of 
$\log \rho_t^{\mathrm{cum}}$ has practical consequences for training, 
and that scaling the clipping thresholds proportionally to $\sqrt{t}$ 
provides a meaningful benefit beyond the theoretical motivation.

\subsection{Training Dynamics}

Figure~\ref{fig:training_dynamics} shows the training dynamics of GRPO, 
GSPO, and CTPO on Qwen3-14B. CTPO achieves higher AIME 2025 accuracy than both baselines throughout 
training, with the performance gap becoming more pronounced in the later stages of training. The clip fraction of CTPO lies between that 
of GRPO and GSPO and remains stable throughout training, indicating 
that position-adaptive clipping does not introduce training instability. All three methods also exhibit similar response length growth, suggesting 
that CTPO does not noticeably alter generation length dynamics.

\begin{figure}[t]
    \centering
    \includegraphics[width=0.32\textwidth]{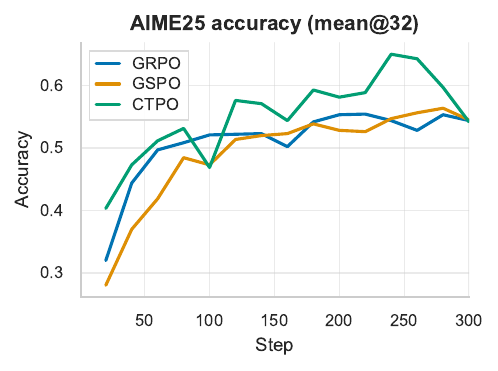}
    \hfill
    \includegraphics[width=0.32\textwidth]{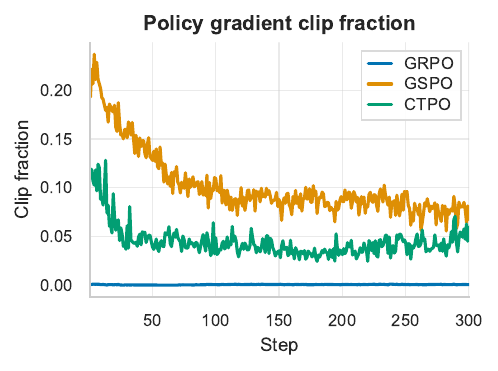}
    \hfill
    \includegraphics[width=0.32\textwidth]{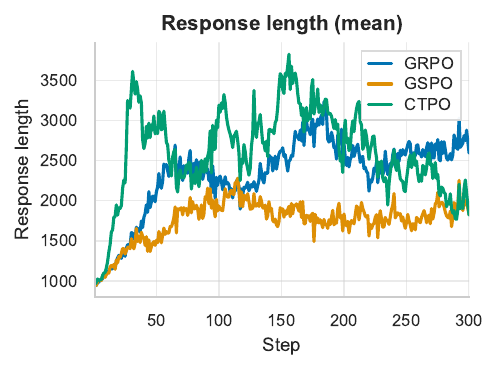}
    \caption{Training dynamics of GRPO, GSPO, and CTPO. 
    \textbf{Left:} AIME 2025 accuracy (avg@32) throughout training. 
    \textbf{Middle:} Policy gradient clip fraction, measured as the fraction of tokens whose method-specific IS ratio exceeds the corresponding clipping threshold. 
    \textbf{Right:} Mean response length throughout training.}
    \label{fig:training_dynamics}
\end{figure}
\section{Related Work}

\paragraph{Importance Sampling in Reinforcement Learning.} 
Importance sampling (IS) has a long history in off-policy reinforcement 
learning as a tool for correcting distributional mismatch between the 
behavior and target policies. Early work by~\citet{precup2000eligibility} 
established foundational IS-based off-policy evaluation methods, 
introducing per-decision importance sampling that exploits the temporal 
structure of MDPs to reduce variance by discarding future terms 
irrelevant to the current decision. Variance reduction 
techniques such as control variates~\citep{greensmith2004variance} 
and weighted importance sampling~\citep{thomas2015safe} have also been 
widely studied in this context. For off-policy value evaluation, 
doubly robust estimators~\citep{jiang2016doubly} combine importance 
sampling with value function estimates to reduce variance while preserving 
unbiasedness under standard conditions. Our work builds on these classical 
ideas and applies them to the specific structure of LLM token generation, 
where the token-level MDP formulation naturally motivates the cumulative 
token IS ratio as a theoretically principled per-decision correction.

\paragraph{Reinforcement Learning for LLM Post-Training.}
Reinforcement learning has become a cornerstone of LLM post-training 
since the advent of RLHF~\citep{ouyang2022training}. PPO~\citep{schulman2017proximal} 
was among the first algorithms applied in this setting to optimize 
a KL-regularized reward objective~\citep{bai2022training}. 
\citet{rafailov2024direct} introduced DPO as a simpler alternative 
that directly optimizes a Bradley-Terry-derived loss without explicit 
reward modeling, spawning a large family of 
variants~\citep{ethayarajh2024kto, meng2024simpo, dong2024rlhf}. 
A parallel line of work relaxes the Bradley-Terry assumption 
entirely~\citep{azar2024general,munos2023nash,ye2024online,wu2024self,zhang2024iterative,zhang2025improving}.
More recently, reinforcement learning has achieved notable success in 
enhancing the reasoning capabilities of LLMs~\citep{jaech2024openai, guo2025deepseek}. 
A representative example is GRPO~\citep{shao2024deepseekmath}, which 
adopts token-level IS ratios and estimates advantages through group-relative 
reward normalization. A number of subsequent works build on the GRPO framework while retaining
token-level IS ratios, improving different components of the training pipeline.
DAPO~\citep{yu2025dapo} introduces dynamic sampling, Dr.~GRPO~\citep{liu2025understanding} corrects length normalization bias, SAPO~\citep{gao2025soft} replaces hard clipping with a soft gating mechanism, and AR3PO~\citep{zhang2025ar3po} improves sampling efficiency via adaptive rollout
and response reuse. In contrast, 
GSPO~\citep{zheng2025group} adopts a length-normalized sequence-level 
ratio, equivalently the geometric mean of per-token ratios, to improve 
numerical stability, but this normalized ratio no longer corresponds to 
the exact full-sequence IS correction. Our work addresses this gap by 
proposing the cumulative token IS ratio, which preserves the exact prefix 
correction needed for token-level policy-gradient terms while avoiding 
unnecessary suffix variance from the full sequence ratio.
\section{Conclusion}\label{sec:conclusion}
We introduced CTPO, which combines the cumulative token IS ratio with position-adaptive clipping to provide principled prefix correction and more consistent regularization, yielding empirical improvements in tool-integrated mathematical reasoning. In the future, we aim to extend CTPO to broader agentic settings involving longer-horizon interaction and richer environment feedback.
\bibliography{ref}

@article{guo2025deepseek,
  title={Deepseek-r1: Incentivizing reasoning capability in llms via reinforcement learning},
  author={Guo, Daya and Yang, Dejian and Zhang, Haowei and Song, Junxiao and Zhang, Ruoyu and Xu, Runxin and Zhu, Qihao and Ma, Shirong and Wang, Peiyi and Bi, Xiao and others},
  journal={arXiv preprint arXiv:2501.12948},
  year={2025}
}

@article{bai2022training,
  title={Training a helpful and harmless assistant with reinforcement learning from human feedback},
  author={Bai, Yuntao and Jones, Andy and Ndousse, Kamal and Askell, Amanda and Chen, Anna and DasSarma, Nova and Drain, Dawn and Fort, Stanislav and Ganguli, Deep and Henighan, Tom and others},
  journal={arXiv preprint arXiv:2204.05862},
  year={2022}
}

@article{schulman2017proximal,
  title={Proximal policy optimization algorithms},
  author={Schulman, John and Wolski, Filip and Dhariwal, Prafulla and Radford, Alec and Klimov, Oleg},
  journal={arXiv preprint arXiv:1707.06347},
  year={2017}
}

@article{shao2024deepseekmath,
  title={Deepseekmath: Pushing the limits of mathematical reasoning in open language models},
  author={Shao, Zhihong and Wang, Peiyi and Zhu, Qihao and Xu, Runxin and Song, Junxiao and Bi, Xiao and Zhang, Haowei and Zhang, Mingchuan and Li, YK and Wu, Yang and others},
  journal={arXiv preprint arXiv:2402.03300},
  year={2024}
}

@article{yu2025dapo,
  title={Dapo: An open-source llm reinforcement learning system at scale},
  author={Yu, Qiying and Zhang, Zheng and Zhu, Ruofei and Yuan, Yufeng and Zuo, Xiaochen and Yue, Yu and Dai, Weinan and Fan, Tiantian and Liu, Gaohong and Liu, Lingjun and others},
  journal={arXiv preprint arXiv:2503.14476},
  year={2025}
}

@article{ouyang2022training,
  title={Training language models to follow instructions with human feedback},
  author={Ouyang, Long and Wu, Jeffrey and Jiang, Xu and Almeida, Diogo and Wainwright, Carroll and Mishkin, Pamela and Zhang, Chong and Agarwal, Sandhini and Slama, Katarina and Ray, Alex and others},
  journal={Advances in neural information processing systems},
  volume={35},
  pages={27730--27744},
  year={2022}
}

@article{rafailov2024direct,
  title={Direct preference optimization: Your language model is secretly a reward model},
  author={Rafailov, Rafael and Sharma, Archit and Mitchell, Eric and Manning, Christopher D and Ermon, Stefano and Finn, Chelsea},
  journal={Advances in Neural Information Processing Systems},
  volume={36},
  year={2024}
}

@article{ethayarajh2024kto,
  title={Kto: Model alignment as prospect theoretic optimization},
  author={Ethayarajh, Kawin and Xu, Winnie and Muennighoff, Niklas and Jurafsky, Dan and Kiela, Douwe},
  journal={arXiv preprint arXiv:2402.01306},
  year={2024}
}

@article{meng2024simpo,
  title={Simpo: Simple preference optimization with a reference-free reward},
  author={Meng, Yu and Xia, Mengzhou and Chen, Danqi},
  journal={Advances in Neural Information Processing Systems},
  volume={37},
  pages={124198--124235},
  year={2024}
}

@article{dong2024rlhf,
  title={Rlhf workflow: From reward modeling to online rlhf},
  author={Dong, Hanze and Xiong, Wei and Pang, Bo and Wang, Haoxiang and Zhao, Han and Zhou, Yingbo and Jiang, Nan and Sahoo, Doyen and Xiong, Caiming and Zhang, Tong},
  journal={arXiv preprint arXiv:2405.07863},
  year={2024}
}

@inproceedings{azar2024general,
  title={A general theoretical paradigm to understand learning from human preferences},
  author={Azar, Mohammad Gheshlaghi and Guo, Zhaohan Daniel and Piot, Bilal and Munos, Remi and Rowland, Mark and Valko, Michal and Calandriello, Daniele},
  booktitle={International Conference on Artificial Intelligence and Statistics},
  pages={4447--4455},
  year={2024},
  organization={PMLR}
}

@article{munos2023nash,
  title={Nash learning from human feedback},
  author={Munos, R{\'e}mi and Valko, Michal and Calandriello, Daniele and Azar, Mohammad Gheshlaghi and Rowland, Mark and Guo, Zhaohan Daniel and Tang, Yunhao and Geist, Matthieu and Mesnard, Thomas and Michi, Andrea and others},
  journal={arXiv preprint arXiv:2312.00886},
  volume={18},
  year={2023}
}

@article{wu2024self,
  title={Self-play preference optimization for language model alignment},
  author={Wu, Yue and Sun, Zhiqing and Yuan, Huizhuo and Ji, Kaixuan and Yang, Yiming and Gu, Quanquan},
  journal={arXiv preprint arXiv:2405.00675},
  year={2024}
}

@article{ye2024online,
  title={Online iterative reinforcement learning from human feedback with general preference model},
  author={Ye, Chenlu and Xiong, Wei and Zhang, Yuheng and Dong, Hanze and Jiang, Nan and Zhang, Tong},
  journal={Advances in Neural Information Processing Systems},
  volume={37},
  pages={81773--81807},
  year={2024}
}

@article{zhang2024iterative,
  title={Iterative nash policy optimization: Aligning llms with general preferences via no-regret learning},
  author={Zhang, Yuheng and Yu, Dian and Peng, Baolin and Song, Linfeng and Tian, Ye and Huo, Mingyue and Jiang, Nan and Mi, Haitao and Yu, Dong},
  journal={arXiv preprint arXiv:2407.00617},
  year={2024}
}

@article{zhang2025improving,
  title={Improving LLM general preference alignment via optimistic online mirror descent},
  author={Zhang, Yuheng and Yu, Dian and Ge, Tao and Song, Linfeng and Zeng, Zhichen and Mi, Haitao and Jiang, Nan and Yu, Dong},
  journal={arXiv preprint arXiv:2502.16852},
  year={2025}
}

@article{jaech2024openai,
  title={Openai o1 system card},
  author={Jaech, Aaron and Kalai, Adam and Lerer, Adam and Richardson, Adam and El-Kishky, Ahmed and Low, Aiden and Helyar, Alec and Madry, Aleksander and Beutel, Alex and Carney, Alex and others},
  journal={arXiv preprint arXiv:2412.16720},
  year={2024}
}

@article{zheng2025group,
  title={Group sequence policy optimization},
  author={Zheng, Chujie and Liu, Shixuan and Li, Mingze and Chen, Xiong-Hui and Yu, Bowen and Gao, Chang and Dang, Kai and Liu, Yuqiong and Men, Rui and Yang, An and others},
  journal={arXiv preprint arXiv:2507.18071},
  year={2025}
}

@article{liu2025understanding,
  title={Understanding r1-zero-like training: A critical perspective},
  author={Liu, Zichen and Chen, Changyu and Li, Wenjun and Qi, Penghui and Pang, Tianyu and Du, Chao and Lee, Wee Sun and Lin, Min},
  journal={arXiv preprint arXiv:2503.20783},
  year={2025}
}

@inproceedings{sheng2025hybridflow,
  title={Hybridflow: A flexible and efficient rlhf framework},
  author={Sheng, Guangming and Zhang, Chi and Ye, Zilingfeng and Wu, Xibin and Zhang, Wang and Zhang, Ru and Peng, Yanghua and Lin, Haibin and Wu, Chuan},
  booktitle={Proceedings of the Twentieth European Conference on Computer Systems},
  pages={1279--1297},
  year={2025}
}

@article{precup2000eligibility,
  title={Eligibility traces for off-policy policy evaluation},
  author={Precup, Doina and Sutton, Richard S and Singh, Satinder},
  year={2000}
}

@article{greensmith2004variance,
  title={Variance reduction techniques for gradient estimates in reinforcement learning},
  author={Greensmith, Evan and Bartlett, Peter L and Baxter, Jonathan},
  journal={Journal of Machine Learning Research},
  volume={5},
  number={Nov},
  pages={1471--1530},
  year={2004}
}

@inproceedings{jiang2016doubly,
  title={Doubly robust off-policy value evaluation for reinforcement learning},
  author={Jiang, Nan and Li, Lihong},
  booktitle={International conference on machine learning},
  pages={652--661},
  year={2016},
  organization={PMLR}
}

@phdthesis{thomas2015safe,
  title={Safe reinforcement learning},
  author={Thomas, Philip S},
  year={2015},
  school={University of Massachusetts Libraries}
}

@article{gao2025soft,
  title={Soft adaptive policy optimization},
  author={Gao, Chang and Zheng, Chujie and Chen, Xiong-Hui and Dang, Kai and Liu, Shixuan and Yu, Bowen and Yang, An and Bai, Shuai and Zhou, Jingren and Lin, Junyang},
  journal={arXiv preprint arXiv:2511.20347},
  year={2025}
}

@article{xue2025simpletir,
  title={Simpletir: End-to-end reinforcement learning for multi-turn tool-integrated reasoning},
  author={Xue, Zhenghai and Zheng, Longtao and Liu, Qian and Li, Yingru and Zheng, Xiaosen and Ma, Zejun and An, Bo},
  journal={arXiv preprint arXiv:2509.02479},
  year={2025}
}

@article{yang2025qwen3,
  title={Qwen3 technical report},
  author={Yang, An and Li, Anfeng and Yang, Baosong and Zhang, Beichen and Hui, Binyuan and Zheng, Bo and Yu, Bowen and Gao, Chang and Huang, Chengen and Lv, Chenxu and others},
  journal={arXiv preprint arXiv:2505.09388},
  year={2025}
}

@misc{deepscaler2025,
  title={DeepScaleR: Surpassing O1-Preview with a 1.5B Model by Scaling RL},
  author={Michael Luo and Sijun Tan and Justin Wong and Xiaoxiang Shi and William Y. Tang and Manan Roongta and Colin Cai and Jeffrey Luo and Li Erran Li and Raluca Ada Popa and Ion Stoica},
  howpublished={\url{https://pretty-radio-b75.notion.site/DeepScaleR-Surpassing-O1-Preview-with-a-1-5B-Model-by-Scaling-RL-19681902c1468005bed8ca303013a4e2}},
  note={Notion Blog},
  year={2025}
}

@article{zhang2025ar3po,
  title={Improving sampling efficiency in rlvr through adaptive rollout and response reuse},
  author={Zhang, Yuheng and Yao, Wenlin and Yu, Changlong and Liu, Yao and Yin, Qingyu and Yin, Bing and Yun, Hyokun and Li, Lihong},
  journal={arXiv preprint arXiv:2509.25808},
  year={2025}
}
\bibliographystyle{plainnat}


\appendix
\section{Proofs for Section~\ref{sec:ctpo}}\label{app:proof}

\subsection{Proof for Proposition \ref{prop:unbiased}} \label{app:proof_1}

\begin{proof}
It suffices to show that each term in the sum is unbiased.
For position $t$, we write $A_t$ as shorthand for $A_t(s_t, a_t)$.
Since the transition dynamics are deterministic, $s_{t'}$ is fully 
determined by the prompt $x$ and the preceding actions $a_{1:t'-1}$ 
for all $t'$. Therefore the prefix likelihood factorizes as 
$\pi_\theta(a_{1:t} \mid x) = \prod_{t'=1}^{t} \pi_\theta(a_{t'} \mid s_{t'})$,
and similarly for $\pi_b$. Applying the change of measure to the prefix:
\begin{align}
    &\mathbb{E}_{\tau \sim \pi_b}\left[\rho_t^{\mathrm{cum}}\, A_t\, 
    \nabla_\theta \log \pi_\theta(a_t \mid s_t)\right] \notag \\
    &= \mathbb{E}_{\tau \sim \pi_b}\left[
    \frac{\pi_\theta(a_{1:t} \mid x)}{\pi_b(a_{1:t} \mid x)}\, 
    A_t\, \nabla_\theta \log \pi_\theta(a_t \mid s_t)\right] \notag \\
    &= \mathbb{E}_{a_{1:t} \sim \pi_\theta(\cdot \mid x)}\left[
    A_t\, \nabla_\theta \log \pi_\theta(a_t \mid s_t)\right], \nonumber
\end{align}
where the last step follows from the change of measure identity, 
and we have marginalized over the suffix $a_{t+1:H}$, 
which does not affect $A_t$ or $\nabla_\theta \log \pi_\theta(a_t \mid s_t)$ 
as both depend only on $(x, a_{1:t})$.
This equals $\mathbb{E}_{\tau \sim \pi_\theta}[A_t \nabla_\theta \log \pi_\theta(a_t \mid s_t)]$,
which is the correct policy gradient term at position $t$.
\end{proof}

\subsection{Proof for Proposition \ref{prop:variance}} \label{app:proof_2}

\begin{proof}
\textbf{(i)}
Decompose the sequence-level ratio as 
$\rho^{\mathrm{seq}} = \rho_t^{\mathrm{cum}} \cdot \epsilon_t$, 
where $\epsilon_t = \prod_{t'=t+1}^{H} r_{t'}$ is the suffix ratio.
By the likelihood ratio identity,
$\mathbb{E}_{\pi_b}[\epsilon_t \mid a_{1:t}] = 1$.
Applying the law of total variance:
\begin{align*}
    \mathrm{Var}(\rho^{\mathrm{seq}}) 
    &= \mathrm{Var}\!\left(\mathbb{E}[\rho_t^{\mathrm{cum}} \epsilon_t \mid a_{1:t}]\right) 
    + \mathbb{E}\!\left[\mathrm{Var}(\rho_t^{\mathrm{cum}} \epsilon_t \mid a_{1:t})\right] \\
    &= \mathrm{Var}(\rho_t^{\mathrm{cum}}) 
    + \mathbb{E}\!\left[(\rho_t^{\mathrm{cum}})^2 
    \mathrm{Var}(\epsilon_t \mid a_{1:t})\right],
\end{align*}
where the second equality uses the fact that $\rho_t^{\mathrm{cum}}$ is 
measurable with respect to $a_{1:t}$ and 
$\mathbb{E}[\epsilon_t \mid a_{1:t}]=1$.
If $\pi_\theta$ and $\pi_b$ differ at some future state reachable from 
$a_{1:t}$ with positive probability under $\pi_b$, then 
$\mathrm{Var}(\epsilon_t \mid a_{1:t})>0$. Therefore, if this occurs on 
a set of prefixes with positive probability, the second term is strictly 
positive, yielding 
$\mathrm{Var}(\rho^{\mathrm{seq}}) > \mathrm{Var}(\rho_t^{\mathrm{cum}})$.

\textbf{(ii)}
Under the independence assumption, the second moment factorizes as:
\begin{equation*}
    \mathbb{E}_{\pi_b}\left[(\rho_t^{\mathrm{cum}})^2\right] 
    = \prod_{t'=1}^{t} \mathbb{E}_{\pi_b}[r_{t'}^2].
\end{equation*}
For each position $t'$, we have
\begin{align*}
    \mathbb{E}_{\pi_b}[r_{t'}^2]
    &=
    \mathbb{E}_{\pi_b}
    \left[
    \sum_{a \in \mathcal{V}}
    \pi_b(a \mid s_{t'})
    \left(
    \frac{\pi_\theta(a \mid s_{t'})}{\pi_b(a \mid s_{t'})}
    \right)^2
    \right] \\
    &=
    1 +
    \mathbb{E}_{\pi_b}
    \left[
    \chi^2(\pi_\theta(\cdot \mid s_{t'}) \| \pi_b(\cdot \mid s_{t'}))
    \right]
    =
    1+\chi^2_{t'}.
\end{align*}
Thus,
\begin{equation*} 
\mathbb{E}_{\pi_b}\left[(\rho_t^{\mathrm{cum}})^2\right] = \prod_{t'=1}^{t} \left(1 + \chi^2_{t'}\right).
\end{equation*}
Since $\mathbb{E}_{\pi_b}[\rho_t^{\mathrm{cum}}]=1$ by the likelihood-ratio identity, subtracting $1$ yields
\[
    \mathrm{Var}(\rho_t^{\mathrm{cum}})
    =
    \prod_{t'=1}^{t} \left(1 + \chi^2_{t'}\right)-1.
\]
Applying the same argument to $\rho^{\mathrm{seq}}$ yields the stated 
expression for $\mathrm{Var}(\rho^{\mathrm{seq}})$.
\end{proof}


\end{document}